\documentclass[conference]{IEEEtran}
\usepackage{times}

\usepackage[numbers]{natbib}
\usepackage{hyperref}

\usepackage{graphicx}
\usepackage{subcaption}
\usepackage{bbm}
\usepackage{amsmath}
\usepackage{amsthm}
\usepackage[singlelinecheck=false,font=footnotesize]{caption}
\usepackage{multirow}
\usepackage{algorithm}
\usepackage{algorithmic}

\newtheorem{lemma}{Lemma}
\newtheorem{theorem}{Theorem}

\newtheorem{assumption}{Assumption}
\DeclareMathOperator*{\argmin}{argmin}

\begin{document}

\title{Pixel-Wise Motion Deblurring of Thermal Videos}

% You will get a Paper-ID when submitting a pdf file to the conference system
% \author{Paper ID: 1238. Author Names Omitted for Anonymous Review.}

% avoiding spaces at the end of the author lines is not a problem with
% conference papers because we don't use \thanks or \IEEEmembership

\author{\authorblockN{Manikandasriram S.R.$^{1}$,
Zixu Zhang$^{1}$,
Ram Vasudevan$^{2}$, and 
Matthew Johnson-Roberson$^{3}$}
\authorblockA{Robotics Institute$^{1}$, Mechanical Engineering$^{2}$, Naval Architecture and Marine Engineering$^{3}$}
\authorblockA{University of Michigan, Ann Arbor, Michigan, USA 48109.}
\authorblockA{\texttt{\{srmani, zixu, ramv, mattjr\}@umich.edu}}
\authorblockA{ \texttt{\url{https://fcav.engin.umich.edu/papers/pixelwise-deblurring}}}
}

\maketitle

\begin{abstract}
Uncooled microbolometers can enable robots to see in the absence of visible illumination by imaging the ``heat'' radiated from the scene. 
Despite this ability to see in the dark, these sensors suffer from significant motion blur. 
This has limited their application on robotic systems. 
As described in this paper, this motion blur arises due to the \emph{thermal inertia} of each pixel.
This has meant that traditional motion deblurring techniques, which rely on identifying an appropriate spatial blur kernel to perform spatial deconvolution, are unable to reliably perform motion deblurring on thermal camera images. 
To address this problem, this paper formulates reversing the effect of thermal inertia at a single pixel as a Least Absolute Shrinkage and Selection Operator (LASSO) problem which we can solve rapidly using a quadratic programming solver. 
By leveraging sparsity and a high frame rate, this pixel-wise LASSO formulation is able to recover motion deblurred frames of thermal videos without using any spatial information.
To compare its quality against state-of-the-art visible camera based deblurring methods, this paper evaluated the performance of a family of pre-trained object detectors on a set of images restored by different deblurring algorithms.
All evaluated object detectors performed systematically better on images restored by the proposed algorithm rather than any other tested, state-of-the-art methods.
\end{abstract}

\IEEEpeerreviewmaketitle

\section{Introduction}
Robots need to operate in a variety of lighting and weather conditions.
Since individual sensing modalities can suffer from specific shortcomings, robotic systems typically rely on multiple modalities~\cite{DBLP:journals/sensors/RosiqueNFP19}.
Visible cameras, for instance, often fail to detect objects hidden in shadows, in poor lighting conditions such as those that arise at night, or those behind sun glare~\cite{DBLP:journals/ral/RamanagopalAVJ18}.
LIDARs, on the other hand, suffer from spurious returns from visual obscurants such as snow, rain, and fog~\cite{DBLP:conf/ivs/BijelicGR18}.

To address these sensor limitations, one can apply thermal infrared cameras that operate in the Long Wave Infrared (LWIR, $7.5\mathrm{\mu m} - 14\mathrm{\mu m}$) region~\cite{DBLP:journals/mva/GadeM14}.
These cameras capture the thermal blackbody radiation emitted by all object surfaces at \emph{``earthly''} temperatures and can therefore operate even in the complete absence of visible light.
Note that these thermal cameras are different from night vision cameras that either operate in Near-Infrared (NIR, $0.75\mathrm{\mu m}-1.4\mathrm{\mu m}$) or use image intensifiers in the visible spectrum~\cite{oatao11729}.
Unlike visible and night-vision cameras that rely on reflected light from a few light sources, thermal cameras depend primarily on emitted radiation, which eliminates shadows and reduces lighting artifacts. 
The longer wavelengths used by these cameras also allow them to see better through visual obscurants~\cite{DBLP:journals/sensors/PericLPV19}.

Despite their potential applicability, these sensors suffer from significant motion blur, as illustrated in Fig.~\ref{fig:pitch_blur}. 
As a result, the research community has focused on scenarios where the relative magnitudes of camera motion, scene geometry and camera resolution can be restricted to limit motion blur \cite{DBLP:journals/tmm/LiuHLZ20, DBLP:journals/tits/ChoiKHPYAK18,DBLP:journals/pr/LiLLZT19,DBLP:conf/cvpr/TreibleSSKOPSK17}.
Unfortunately, such restrictions can be difficult to accommodate during mobile robot applications.
In addition, the inability to control exposure time and the lack of a global shutter make microbolometers even more challenging to use. 

\begin{figure}[t]
    \centering
    \begin{subfigure}{\linewidth}
        \centering   
        \includegraphics[width=\linewidth]{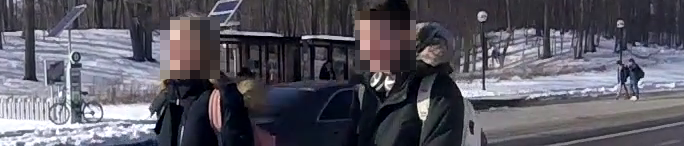}
    \end{subfigure}\\[5pt]
    \begin{subfigure}{\linewidth}
        \centering
        \includegraphics[width=\linewidth]{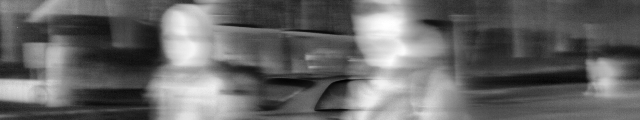}
    \end{subfigure}\\[5pt]
    \begin{subfigure}{\linewidth}
        \centering
        \includegraphics[width=\linewidth]{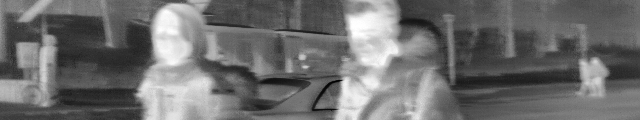}
    \end{subfigure}
    \caption{An illustration of the proposed motion deblurring algorithm for microbolometers. The top image shows a visible image captured at $30\mathrm{fps}$  with auto exposure. The middle image shows a thermal image captured at $200\mathrm{fps}$ which suffers from significant non-uniform motion blur. As described in this paper, the lack of exposure control and the difference in the physics of image formation aggravates the image degradation due to motion blur in thermal images. The bottom image shows the result from our algorithm that processes each pixel independently. Our model-based approach is able to eliminate the blur without any visible artefacts.}
    \label{fig:pitch_blur}
\end{figure}

To address these motion blur problems, one can rely on another class of thermal cameras called cooled photon detectors~\cite{rogalski2012history}.
These sensors use the photoelectric effect, i.e. the same physics underlying CCD and CMOS sensors, which enables controllable exposure time and global shutter and, as a result, limits motion blur.
However, these cooled photon detectors are prohibitively expensive since they require cooling the sensor to $77\mathrm{K}$ to eliminate self-emissions~\cite{barela2017comparison}.
Therefore they are used only in high-end applications. 

Uncooled microbolometers, on the other hand, due to their affordability, low weight and low power consumption, have brought thermal imaging to consumer applications.
As a result, they are the predominant type of thermal cameras used by the robotics and vision research community today~\cite{DBLP:journals/mva/GadeM14}.
These sensors use the bolometer principle, where the incoming radiation heats up a detector material that has a temperature dependent electrical resistance~\cite{vollmer2017infrared}. 
Unlike the photoelectric effect, temperature changes take a non-negligible amount of time, which leads to a lag in these cameras with respect to changes in the scene.
This paper illustrates that this lag, which we refer to as \emph{thermal inertia}, gives rise to motion blur in thermal images.
For applications with continuous motion, such as autonomous vehicles, the sensors never reach steady state.
The measured temperatures might then be incorrect, making these images an inaccurate representation of the scene.
Unfortunately, designing microbolometers with lower thermal inertia without increasing sensor noise remains a significant technological challenge~\cite{boudou2019ulis}.
Therefore, motion deblurring algorithms are required to accurately recover the instantaneous scene representation from the data measured by these microbolometers.

There has been limited success in developing general motion deblurring algorithms for microbolometers~\cite{vollmer2017infrared}.
Motion blur is typically modelled as a latent sharp image convolved with an unknown blur kernel, which can vary across the image. 
As a result, motion deblurring requires estimating the blur kernel and using deconvolution to recover the latent sharp image.
For instance,  \citet{oswald2018temperature} construct an exponential blur kernel using the thermal time constant for uniform motions of a target object against a constant background. 
\citet{nihei2019simple} use an Inertial Measurement Unit (IMU) to estimate the camera motion and provide an approximate correction equation to recover the steady state value of a target pixel.
In addition, one could consider applying algorithms developed for visible cameras to thermal images~\cite{DBLP:books/cu/RH2014}. 
For example, \citet{DBLP:conf/cvpr/YanRGWC17} formulate deblurring as a maximum \textit{a posteriori} problem that alternatively optimizes the latent image and the blur kernel with an additional image prior, which promotes sparsity of image gradients.
\citet{DBLP:journals/ijcv/WhyteSZP12} consider non-uniform blurring caused due to 3D rotation of the camera against a static scene.
\citet{DBLP:conf/iccv/BahatEI17} estimate per-pixel blur kernels to accommodate dynamic objects in the scene.
More recently, there have been a number of learning based methods proposed for single-image deblurring~\cite{DBLP:conf/cvpr/NahKL17, DBLP:conf/cvpr/TaoGSWJ18,DBLP:conf/iccv/KupynMWW19} and multi-image/video deblurring~\cite{DBLP:conf/iccv/ZhouZPZXR19,DBLP:conf/cvpr/WangCYDL19}.

This paper diverges from these conventional deblurring approaches that perform kernel estimation and subsequent latent image recovery, and instead focuses on devising a method that can correct for the microbolometer-specific cause for motion blur.
In particular, we show that the lag introduced by the bolometer principle at each individual pixel directly gives rise to the motion blur. 
By reversing this effect of thermal inertia pixel-wise without any spatial constraints, from a sequence of high frame rate measurements, we recover motion deblurred images without explicitly modeling the motion of the camera, scene depth, or scene dynamics. 
To this end, we construct an under-determined system of linear equations, which has infinitely many solutions, by discretizing at a high rate the differential equation that models the thermal image formation process. 
We then leverage temporal sparsity to select a unique solution while being robust to measurement noise by using the Least Absolute Shrinkage and Selection Operator (LASSO) and solve it using quadratic programming solvers.
Our formulation works across a variety of scenes with arbitrary 6D relative motions between the camera and the dynamic objects within the scene.
In addition, our method is uniquely suited for robotic vision tasks where an image is often represented by a sparse set of keypoints. 
While existing approaches at-least require deconvolution over small patches covering the sparse keypoints, our proposed approach just requires the data from the sparse pixels of interest. 

The key contributions of our paper are summarized below:
\begin{itemize}
    \item In Section \ref{sec:inverting_hysteresis}, we show that reversing the effect of thermal inertia at an independent pixel can be formulated as solving an under-determined system of linear equations.
    \item In Section \ref{sec:deblurring_equivalence}, we show that assuming temporal sparsity of the signal pixel-wise is sufficient to recover motion deblurred frames from thermal videos using LASSO without requiring any training data.
    \item In Section \ref{sec:results}, we illustrate that our model based approach outperforms state-of-the-art learning based single-image and multi-image deblurring algorithms. In particular, we show that a variety of object detectors performed systematically better when using the images deblurred by our proposed algorithm rather than any of the five tested, state-of-the-art methods.
\end{itemize}
The remainder of our paper is organized as follows.
Section \ref{sec:imaging_models} reviews the origins of motion blur in visible cameras and microbolometers based on their physics of image formation.
Our key contributions are presented in Sections \ref{sec:inverting_hysteresis}, \ref{sec:deblurring_equivalence} and \ref{sec:results}.
Finally, Section \ref{sec:discussion} concludes the paper with a brief discussion on the insights offered by our proposed framework.

% This demo file is intended to serve as a ``starter file" for the
% Robotics: Science and Systems conference papers produced under \LaTeX\
% using IEEEtran.cls version 1.7a and later.  
% Discussing related works in isolation is hard for this paper since many concepts need to be introduced before they can be referenced
\section{Image Formation Models}
\label{sec:imaging_models}

This section reviews the underlying physics of image formation to identify the cause of motion blur as summarized in Fig.~\ref{fig:image_formation_model}.
In particular, we focus on two distinct Image Formation Models (IFM): the Photoelectric IFM for visible cameras and the Microbolometer IFM for thermal cameras. 
We present both models to emphasize the distinct reasons that give rise to motion blur in each IFM.

\begin{figure}[t]
    \includegraphics[width=\linewidth]{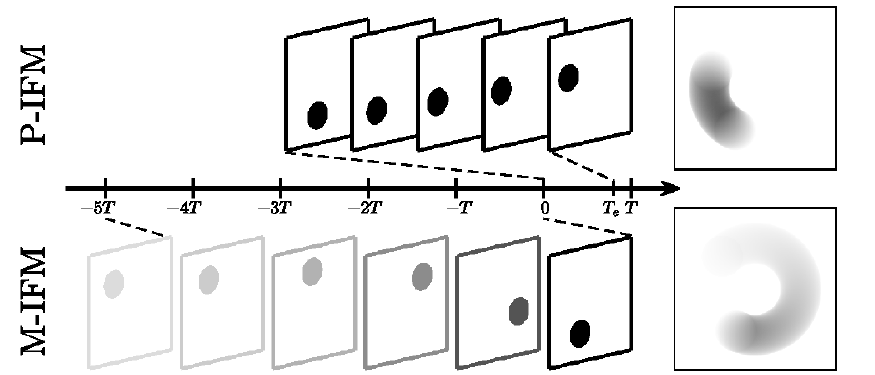}
    \caption{An illustration of motion blur in photoelectric sensors (top) and microbolometers (bottom). 
    Each frame on the left represents a time slice of instantaneous power from the scene. 
    P-IFM averages over time slices during exposure while M-IFM gives more weight to recent frames as depicted by the transparency in this figure. 
    The blurred images on the right were generated by pixelwise average for P-IFM and by simulating a microbolometer with $\tau=T$ for M-IFM. 
    Note that the blur looks symmetric for P-IFM since each slice is equally weighted while the motion blur in M-IFM appears as a blur-\textit{tail}.}
    \label{fig:image_formation_model}
\end{figure}

\subsection{Photoelectric Image Formation Model}
\label{subsec:photoelectric_model}

In Photoelectric IFM (P-IFM), the sensor is exposed for a finite duration to incoming Electro-Magnetic (EM) radiation from the scene and the number of photons ``collected'' is digitized into an image. 
The registers are reset to zero before the next frame is captured.

Let $\Phi(p)$ denote the total power from a point $p$ in the scene 
and $T_e$ denote the exposure time.
In the case of global shutter and a linear camera response function, the value of pixel $(i, j)$ can be written as
\begin{equation}
    I(i, j) = \frac{1}{T_e} \int_{0}^{T_e} \Phi(p_{i,j}(s)) ds,
    \label{eq:photoelectric_ifm}
\end{equation}
where $p_{i,j}(s)$ is the scene point being imaged by pixel $(i,j)$ at time $s$.
In the absence of relative motion during exposure, $p_{i,j}$ corresponds to a single scene point and the integrand above simplifies to a constant.

When there is relative motion during exposure, $p_{i,j}$ instead corresponds to a sequence of scene points, resulting in motion blur. 
To describe this more explicitly, let $\delta s \in \mathbbm{R}$ with $\frac{T_e}{\delta s} \in \mathbbm{N}$, and suppose we divide the total exposure time into $\frac{T_e}{\delta s}$ intervals each of length $\delta s$, such that for each $k \in \{0,\ldots, \frac{T_e}{\delta s} - 1\}$, $p_{i,j}(s) = p_{i,j}^k$ is constant for all $s \in [k \delta s, (k + 1) \delta s)$. 
Then
\begin{equation}
    I(i, j) = \frac{1}{T_e} \sum_{k=0}^{\frac{T_e}{\delta s} - 1} \Phi(p_{i,j}^k)\delta s.
\end{equation}

The objective of motion deblurring algorithms is to recover an image of the scene at the start of exposure or equivalently $\Phi(p_{i,j}^0)$. 
Over the duration of exposure, the relative motion causes different pixels to image a single scene point. 
Therefore, the scene point imaged by $(i,j)$ at the $k$-th interval would have been imaged by pixel $(x,y)$ at the start of exposure.
In other words, every $k$-th slice can be obtained by warping $\Phi(p_{i,j}^0)$.
In the computer vision literature~\cite{DBLP:books/cu/RH2014}, $\Phi(p_{i,j}^0)$ is referred as the latent sharp image, $L$, and the blurred image is written as 2D convolution
\begin{equation}
    I(i,j) = \iint H(i-x, j-y)L(x, y)\, dx\, dy,
\end{equation}
where $H$ is referred to as a Point Spread Function (PSF) or blur kernel. 
The blur kernel models how much time pixel $(i,j)$ was exposed to the scene point that would have been imaged at $(x,y)$ at the start of the exposure.
As a result, this explicitly models the relative motion and scene dynamics. 

\subsection{Microbolometer Image Formation Model}
\label{subsec:microbolometer_model}

In Microbolometer IFM (M-IFM), the sensor is constantly exposed to incoming EM radiation from the scene.
The resulting change in the temperature of a pixel changes its electrical resistance.
A Read-Out Integrated Circuit (ROIC) measures the current electrical resistance at regular intervals and computes corresponding temperature values to be returned by the camera.
When the incoming radiation is continuously changing, the microbolometer is unable to reach an
equilibrium before the ROIC takes the next measurement. 
Moreover, there is no ``reset'' mechanism at the end of a frame. 
The microbolometer can be modeled by \cite{boudou2019ulis, topaloglu2009analysis, vollmer2017infrared}
\begin{equation}
    \frac{C_{th}}{G_{th}}\frac{d \Delta T}{dt}(t) + \Delta T(t) = \frac{\alpha}{G_{th}} \Phi(p(t)),
\label{eq:microbolometer_heat_eqn}
\end{equation}
where $C_{th}$, $G_{th}$ and $\alpha$ are constants representing the thermal capacitance, thermal conductance, and absorptivity of the microbolometer respectively, with $\Phi$ and $p$ as in the previous subsection.
$\Delta T(t) = T_m(t) - T_s$ is the difference in temperature at time $t$ between the pixel membrane $T_m$ and a constant substrate temperature $T_s$. 

We can define the steady state temperature of the pixel membrane $T_m^{ss}(p(t))$ corresponding to the power $\Phi(p(t))$ as 
\begin{equation}
    T_m^{ss} (p(t)) = T_s + \frac{\alpha}{G_{th}} \Phi(p(t)).
    \label{eq:T_mss_to_Phi}
\end{equation}
Inserting in \eqref{eq:microbolometer_heat_eqn}, we get
\begin{equation}
    \tau \frac{d T_m}{dt}(t) + T_m(t) = T_m^{ss}(p(t)),
    \label{eq:microbolometer_rc_form}
\end{equation}
where $\tau = \frac{C_{th}}{G_{th}}$ is the thermal time constant of the microbolometer.
The general solution to \eqref{eq:microbolometer_rc_form} can be obtained by integration-by-parts as:
\begin{equation}
    T_{m}(t) = \frac{1}{\tau} \int_{-\infty}^{t} e^{\frac{s-t}{\tau}} T_{m}^{ss}(p(s)) ds.
    \label{eq:microbolometer_gen_soln}
\end{equation}
Note that \eqref{eq:microbolometer_gen_soln} is the solution of a first order linear differential equation, and $T_m^{ss}$ is proportional to $\Phi$ from \eqref{eq:T_mss_to_Phi}. 
As a result, a step change in $\Phi$ would require a delay of $5\tau$ before $T_m$ reaches within $1\%$ of the steady state value. 
This lag in the temperature of the microbolometer $T_m$ with respect to changes in the incoming power $\Phi$ is what we call the \textit{thermal inertia} of each pixel.

Every pixel follows the same model and hence the value of pixel $(i,j)$ returned by the camera at time $t$ can be written as 
\begin{equation}
    I(i, j) = \frac{1}{\tau} \int_{-\infty}^{t} e^{\frac{s-t}{\tau}} T_m^{ss}(p_{i,j}(s)) ds.
    \label{eq:microbolometer_ifm}
\end{equation}
Since $T_m^{ss}$ is proportional to $\Phi$, as a result of \eqref{eq:T_mss_to_Phi}, the integral in \eqref{eq:microbolometer_ifm} can be interpreted as a weighted sum of the incoming power over the history of the pixel with the weight exponentially decaying into the past.

In the absence of relative motion over the last $5\tau$ seconds, the pixel reaches steady state such that $I(i,j) \approx T_m^{ss}(p_{i,j}(t))$. 
However, when there is relative motion, the value of a pixel depends on $p_{i,j}$, resulting in apparent motion blur. 
The objective of motion deblurring for thermal images is to recover $T_m^{ss}(p_{i,j}(t))$ for each $i$ and $j$.

\subsection{Spatial vs Temporal Spread}
\label{subsec:spatial_vs_temporal}

In P-IFM, the exposure time is divided into small intervals within which the incoming power is described by a single scene point.
During motion, a single pixel's intensity is described by the power of a variety of points in the world. 
To identify a latent sharp image, one can identify a blur kernel and apply spatial deconvolution.
One could follow a similar approach for \eqref{eq:microbolometer_ifm} wherein the blur kernel would additionally need to model the time dependent exponential weight factor and the time constant.
The effective exposure time that needs to be considered would also be larger than the typical values of $T_e$.
In addition, the challenges associated with estimating $H$ to describe the 6D relative motion in dynamic scenes would only further aggravate the problem.

This work takes a different tact. 
The reset mechanism at the end of a frame in P-IFM renders each frame as an independent measurement of the scene.
However, two consecutive frames taken at $nT$ and $(n+1)T$ in M-IFM have a common history $[(n+1)T-5\tau, nT]$ where $T$ is the sampling period of the camera and is chosen such that $T<5\tau$.
For instance, if $T=\tau$, then $5$ consecutive measurements have different portions of $T_m^{ss}(p(nT))$ in it.
Suppose we have an algorithm that can undo this temporal spread to recover $T_m^{ss}(p(nT))$ given a set of measurements $T_m(nT)$, then such an algorithm achieves motion deblurring for thermal images without modeling the relative motion of the camera with the scene. 
The next section poses this inverse problem as selecting a unique signal from the solutions of an under-determined system of linear equations. 
\section{Reversing the Effect of Thermal Inertia}
\label{sec:inverting_hysteresis}

To understand how to reverse the thermal inertia in microbolometers, lets consider the following differential equation
\begin{equation}
    \tau \frac{dy}{dt}(t) + y(t) = x(t),
    \label{eq:model_ode}
\end{equation}
where $\tau$ is a known constant, $x$ is the unknown signal of interest and $y$ is the signal from which we have discrete samples $\{y(t_0), y(t_1), \ldots, y(t_N)\}$.
Our objective is to recover the signal $x$ within the right open interval $[t_0, t_N)$ that respects \eqref{eq:model_ode} and the given samples of $y$.
Note that \eqref{eq:model_ode} is equivalent to \eqref{eq:microbolometer_rc_form} from previous section.

By integrating \eqref{eq:model_ode}, we get for $t_0 \leq t \leq t_N$
\begin{equation}
    y(t) = y(t_0) e^{-\frac{t-t_0}{\tau}} + \frac{e^{-\frac{t}{\tau}}}{\tau}\int_{t_0}^{t} e^{\frac{\alpha}{\tau}}x(\alpha) d\alpha.
    \label{eq:model_ode_soln}
\end{equation}
This gives rise to $N$ constraints, one for each $t_i \in \{t_1, \ldots, t_N\}$
\begin{equation}
    y(t_i) = y(t_0) e^{-\frac{t_i-t_0}{\tau}} + \frac{e^{-\frac{t_i}{\tau}}}{\tau}\int_{t_0}^{t_i} e^{\frac{\alpha}{\tau}}x(\alpha) d\alpha.
    \label{eq:constraints}
\end{equation}
This is an under-constrained problem as there are many signals $x$ that satisfy \eqref{eq:constraints}.
Therefore, we proceed by constructing a feasible set of such signals.

Let the indicator function be denoted as
\begin{equation}
    \mathbbm{I}(x_{\min} \leq x < x_{\max}) = 
        \left\{
    	    \begin{array}{ll}
        		1  & \mbox{if } x_{\min} \leq x < x_{\max} \\
    	    	0 & \mbox{otherwise}
    	    \end{array}
        \right.
    \label{eq:indicator_fn}
\end{equation}
Consider the class of piece-wise constant signals $f_n : [t_0, t_N) \xrightarrow{} \mathbbm{R}$ for $n \in \mathbbm{Z}$ defined by
\begin{equation}
\label{eq:functionspace}
    f_n(\alpha) = \sum_{k=0}^{K_n-1} a_k \mathbbm{I}\left(t_0 +\frac{k\Delta t}{K_n} \leq \alpha < t_0 + \frac{(k+1) \Delta t}{K_n}\right), 
\end{equation}
where $K_n=2^n$ and $\Delta t = t_N - t_0$.
This function can be equivalently represented by the vector $A  = [a_0, a_1, \ldots, a_{K_n-1}]^T \in \mathbbm{R}^{K_n}$.
Note that the set of functions that can be represented by vectors drawn from $\mathbbm{R}^{K_n}$ is a subset of the set of functions that can be represented by vectors drawn from $\mathbbm{R}^{K_m}$ when $m \geq n$.

Next, consider the following important property of the class of functions we have just defined that follows from \cite[Theorem 2.26]{folland2013real}:
\begin{lemma}

Let $g: [t_0, t_N) \xrightarrow{} \mathbbm{R}$ be a Lebesgue integrable function.
Then for every $\epsilon > 0$, there exists $n \in \mathbbm{Z}$ and function $f_n$ as in \eqref{eq:functionspace} with $\|g-f_n\|_1 < \epsilon$.

\label{lem:x_to_A}
\end{lemma}

The above lemma allows us to associate any signal $x$ within $[t_0, t_N)$ with a function $f_n$ within machine precision for a sufficiently large $n$.
This allows us to construct our feasible set of signals as a subset of $\mathbbm{R}^{K_n}$

Next, we describe how to formulate the set of $N$ constraints in \eqref{eq:constraints} as a set of linear equations:
\begin{theorem}

For each $m \in \mathbbm{Z}$, consider $f_m$ as in \eqref{eq:functionspace} with a vector of coefficients $A_{f_m} \in \mathbbm{R}^{K_m}$.
If $f_m$ satisfies \eqref{eq:constraints}, then $A_{f_m}$ satisfies the following equation:
\begin{equation}
    Y = V A_{f_m},
    \label{eq:thm_A_space}
\end{equation}
where $Y \in \mathbbm{R}^N$ and $V \in \mathbbm{R}^{N \times K_m}$ are defined as 
    \begin{subequations}
    \begin{align}
        [Y]_j &= y(t_{j+1}) - y(t_0) e^{-\frac{t_{j+1}-t_0}{\tau}} \\
        [V]_{j,k} &= e^{-\frac{t_{j+1}-t_0}{\tau}} (e^{\gamma_{\max}}-e^{\gamma_{\min}})\mathbbm{I}(\gamma_{\min} < \gamma_{\max}),
    \end{align}
    \end{subequations} 
    where $ 0 \leq j < N,  0 \leq k < K_m$, the subscripts outside the square-brackets represent the corresponding component in the matrix,  and 
         \begin{subequations}
    \begin{align}
        \gamma_{\min} &= \max\{\frac{k}{K_m}\frac{\Delta t}{\tau}, 0\} \\
        \gamma_{\max} &= \min\{\frac{k+1}{K_m}\frac{\Delta t}{\tau}, \frac{t_{j+1}-t_0}{\tau}\}.
    \end{align}
     \end{subequations}
\label{thm:feasible_set}
\end{theorem}
\begin{proof}
Following Lemma~\ref{lem:x_to_A}, there exists $n$ such that we can substitute for $x$ in \eqref{eq:model_ode_soln} and using change of variables $\beta = \frac{\alpha-t_0}{\Delta t}$ and using $s=t-t_0$, we get

\begin{equation}
    y(t) = y(t_0) e^{-\frac{s}{\tau}} + \frac{\Delta t}{\tau}e^{-\frac{s}{\tau}}\int_{0}^{\frac{s}{\Delta t}} e^{\frac{\beta \Delta t}{\tau}}f_n(t_0 + \beta\Delta t) d\beta.
\end{equation}

Substituting for $f_n$, the second term in the above equation becomes

\begin{equation}
    \sum_{k=0}^{K_n-1} a_k e^{-\frac{s}{\tau}}  \int_{0}^{\frac{s}{\Delta t}} e^{\frac{\beta \Delta t}{\tau}} \mathbbm{I}(\frac{k}{K_n} \leq \beta < \frac{k+1}{K_n}) \frac{\Delta t}{\tau} d\beta.
\end{equation}
Using another change of variables $\gamma = \frac{\beta \Delta t}{\tau}$,
\begin{equation}
    \sum_{k=0}^{K_n-1} a_k e^{-\frac{s}{\tau}} \Big( \int_{\gamma_{\min}^{n,k,s}}^{\gamma_{\max}^{n,k,s}} e^{\gamma} d\gamma \Big) \mathbbm{I}(\gamma_{\min}^{n,k,s} < \gamma_{\max}^{n,k,s}),
\end{equation}
where 
\begin{subequations}
    \begin{align}
    \gamma_{\min}^{n,k,s} &= \max\{\frac{k}{K_n}\frac{\Delta t}{\tau}, 0\} \\
    \gamma_{\max}^{n,k,s} &= \min\{\frac{k+1}{K_n}\frac{\Delta t}{\tau}, \frac{s}{\tau}\}.
    \end{align}
\end{subequations}
The solution to the integral can now be written in closed form yielding 
\begin{equation}
\begin{split}
    y(t) - & y(t_0) e^{-\frac{s}{\tau}}  = \\  & \sum_{k=0}^{K_n-1} a_k e^{-\frac{s}{\tau}} (e^{\gamma_{\max}^{n,k,s}}-e^{\gamma_{\min}^{n,k,s}}) \mathbbm{I}(\gamma_{\min}^{n,k,s} < \gamma_{\max}^{n,k,s}).
    \label{eq:ode_soln_sum}
\end{split}
\end{equation}

Substituting for $t$ with values $t_1, t_2, \ldots, t_N$, we retrieve \eqref{eq:thm_A_space} with $m=n$.
Therefore, for every $f_n$ satisfying \eqref{eq:constraints}, the corresponding $A_{f_n}$ satisfies \eqref{eq:thm_A_space} by construction.

\end{proof}

The above theorem allows us to represent the constraints in \eqref{eq:constraints} as a system of linear equations \eqref{eq:thm_A_space}.
The choice of $n$ divides the interval $[t_0, t_N)$ into uniform intervals of size $\frac{\Delta t}{2^n}$ with discontinuities at integral multiples of $\frac{\Delta t}{2^n}$.
Therefore, we align the observation time instances $\{t_0, t_1, \ldots, t_N\}$, which are typically uniformly spaced, with the discontinuities of $f_n$ by choosing $N$ to be a power of $2$. 
Recall that increasing $n$ strictly increases the set of signals that are represented.
Therefore, the set of all possible feasible solutions to \eqref{eq:constraints} is obtained by letting $n \xrightarrow{} \infty$. 
The next sections shows that temporal sparsity can be leveraged to select a unique signal from this feasible set.
\section{Motion Deblurring using Temporal Sparsity}
\label{sec:deblurring_equivalence}

Recall from Section~\ref{subsec:microbolometer_model} that motion blur arises from the thermal inertia of each pixel.
In the previous section, we saw that there are infinitely many signals that agree with the data returned by the camera as well as the camera parameter $\tau$.
Yet, following Section~\ref{subsec:spatial_vs_temporal}, we need to choose a unique signal from this feasible set using only that pixel's temporal data. 
We address this by pruning the feasible set based on two observations presented here.

To understand our first observation, consider the time between two consecutive observations $[t_i, t_{i+1}]$.
A large value of $n$ would divide this interval into many small intervals, each of length $\frac{\Delta t}{2^n}$.
Suppose we consider a function $f_n: [t_0,t_N) \to \mathbbm{R}$ as defined in \eqref{eq:functionspace} and its corresponding vector representation $A = [a_0,\ldots,a_{K_n-1}]^T \in \mathbbm{R}^{K_n}$.
Note when $n$ is large, any pair $a_p$ and $a_q$
corresponding to intervals between $[t_i,t_{i+1}]$
can be increased in equal and opposite directions while effectively cancelling each other out at $t_{i+1}$.
Therefore, given a signal with representation $A$ satisfying \eqref{eq:constraints}, there are many pairs of coefficients that can be changed when $n$ is large to obtain new signals that satisfy \eqref{eq:constraints}. 

Next, to describe our second observation consider a typical scene to be imaged.
The surfaces in the scene are locally at similar temperatures with distinct jumps at object or material boundaries that appear as edges in the thermal image.
For most applications, these edges capture sufficient information in the image.
The trajectory of scene points imaged by a single pixel would have temperatures that change drastically only at the few instances when the imaged scene point passes through a boundary. 
Based on the above two observations, we propose the following assumption about the signal of interest:
\begin{assumption}
    The signal of interest, $T_m^{ss}$, can be approximated by a piecewise constant signal that satisfies \eqref{eq:constraints} with the minimum number of transitions between consecutive components of its representation vector. 
    \label{asm:temporal_sparsity}
\end{assumption}
To understand this assumption, recall that as a result of Lemma \ref{lem:x_to_A} any integrable function can be approximated arbitrarily well using piecewise constant functions as the number of discontinuities is allowed to go to infinity.
Reducing the number of such discontinuities would discourage redundant transitions from coefficients $a_p$ to $a_q$ as in the first observation. 
By minimizing the number of step changes, we also leverage the second observation that the temperature of a pixel in a thermal image does not typically have many large changes as a function of time. 

\begin{algorithm}[t]
\caption{Pixelwise Deblurring using QP}
\label{alg:summary}
\begin{algorithmic}
\REQUIRE{$\{t_0, \ldots, t_N\}, \{ y(t_0), \ldots, y(t_N)\}$, $T$, $\log_2(N) \in \mathbbm{N}$, $\lambda$}
\ENSURE{$A^* \in \mathbbm{R}^{K_n} \iff x : [t_0, t_N) \xrightarrow{} \mathbbm{R}$}
\STATE Choose $n$ such that $\frac{\Delta t}{2^n} \ll T$
\STATE Construct $V,Y$ from \eqref{eq:thm_A_space}
\STATE Construct Haar Matrix $H$ of size $K_n \times K_n$
\STATE Setup objective to minimizing $\|Y-VH^TD\|_2 + \lambda \|D\|_1$
\STATE Solve quadratic program to obtain $D^*$
\RETURN $A^* = H^TD^*$
\end{algorithmic}
\end{algorithm}

Computing the signal of interest that satisfies this assumption can be formulated as the solution to the following optimization problem:
\begin{align}
    A^* = \argmin_{A \in \mathbbm{R}^{K_n} } &~ \rho(A) \label{eq:noncvx_optim} \\
    \textrm{s.t.} \quad VA &= Y \nonumber,
\end{align}
where $A^*$ denotes the vector representing the approximation to the signal of interest that satisfies Assumption \ref{asm:temporal_sparsity}, $V,Y$ and $K_n$ are as in Theorem~\ref{thm:feasible_set} for a large value of $n$, and $\rho : \mathbbm{R}^{K_n} \xrightarrow{} \mathbbm{W}$ counts the number of transitions between consecutive components of the representation vector $A$ (i.e. $\sum_{k=1}^{K_{n}-1} \mathbbm{I}\{a_{k-1} \neq a_{k}\}$).
This is a non-convex optimization problem. 
In this paper, we instead solve a closely related convex optimization problem described below.

Let $D$ denote the coefficients of $A$ in the Haar system such that $D = HA$ where $H$ is the Haar matrix of size $K_n \times K_n$\cite{chui2016introduction}.
Note that a signal with few transitions would also be sparse in the Haar system.
Therefore, we can relax \eqref{eq:noncvx_optim} to
\begin{equation}
    D^* =  \argmin_{D\in \mathbbm{R}^{K_n}} ~ \|VH^TD - Y\|_2 + \lambda \|D\|_1  ,
\label{eq:qp_program} 
\end{equation}
where $\lambda$ is a hyper-parameter that balances between respecting \eqref{eq:thm_A_space} and temporal sparsity from Assumption~\ref{asm:temporal_sparsity}.
In particular, we have relaxed the equality constraint in \eqref{eq:noncvx_optim} to account for measurement noise.
This optimization problem is called a LASSO problem \cite{tibshirani1996regression} and can be solved efficiently using a Quadratic Programming (QP) solver~\cite{cplex}. 
Algorithm~\ref{alg:summary} summarizes the proposed approach to deblur a single pixel.

\section{Experimental Results}
\label{sec:results}

The experimental results are divided into five subsections. 
First, we validate our Microbolometer Image Formation Model. 
We then illustrate the proposed approach using synthetic data for a single pixel.
Next, we introduce the datasets and baselines presented in the results.
Next, we present deblurring results for typical outdoor scenes.
Finally, we present quantitative evaluation of state-of-the-art object detectors on images restored using different deblurring algorithms.

All the experiments were carried out using a radiometrically calibrated FLIR A655sc camera capturing frames every $5\mathrm{ms}$ at $640\times120$ resolution.
The object parameters were set to unity for both emissivity and atmospheric transmissitivity.
The following parameters were used throughout the experiments: $\tau = 11\mathrm{ms}$, $N=17$, $n=7$ and $\lambda=0.5$.
Recall that the proposed approach can deblur a frame using any consecutive set of $N+1$ measurements of which it is a part. 
In our experiments, the past $N$ frames were used, unless noted otherwise, so that future information is not required. 
The quadratic programs were solved using IBM CPLEX v12.10 through the Python API~\cite{cplex}.
On average over $20$ clips of $17$ frames each, processing a single pixel took $212\mathrm{ms}$ including problem setup on an Intel Xeon 8170M server processor running at $2.1\mathrm{GHz}$. 
The running time can be further reduced by using optimised code with the C++ API.

\subsection{Validating M-IFM}
\label{subsec:validate_model}

\begin{figure}[t]
    \centering
    \includegraphics[width=\linewidth]{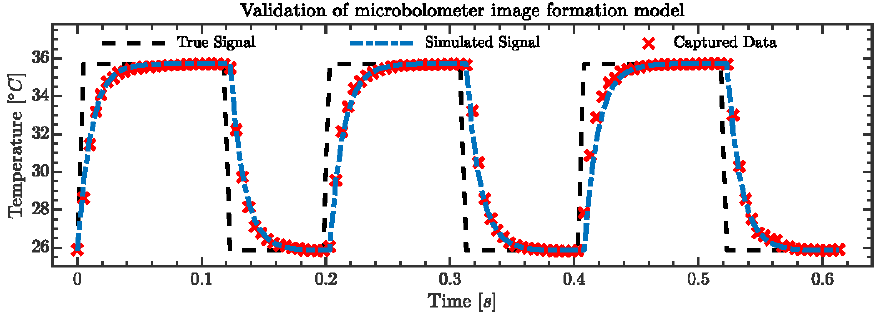}
    \caption{Temporal plot of a single pixel from a real-world microbolometer when subjected to a series of rectangular pulses. The simulated signal for $\tau=11\mathrm{ms}$ agrees well with the captured real-world data.}
    \label{fig:validate_single_pixel}
\end{figure}

Microbolometers utilise a number of internal signal processing algorithms to compute the temperature values that we receive as an image.
Despite this, we show here that the overall response of a microbolometer to changes in scene still follows \eqref{eq:microbolometer_rc_form}.
To this end, we setup the camera to view a constant warm surface, such as a powered on monitor and move a board at a cooler temperature in front of it repeatedly such that an individual pixel experiences a series of rectangular pulses.
We manually approximated the rectangular pulse and fit the M-IFM to the captured data to get $\tau=11\mathrm{ms}$. Fig.~\ref{fig:validate_single_pixel} illustrates good agreement between the captured data and the signal simulated using $\tau=11\mathrm{ms}$ thus validating our M-IFM.
We use this value of $\tau$ in the rest of the experiments.

\subsection{Deblurring at a single pixel}
\label{subsec:single_pixel}

\begin{figure}[t]
    \centering
        \includegraphics[width=\linewidth]{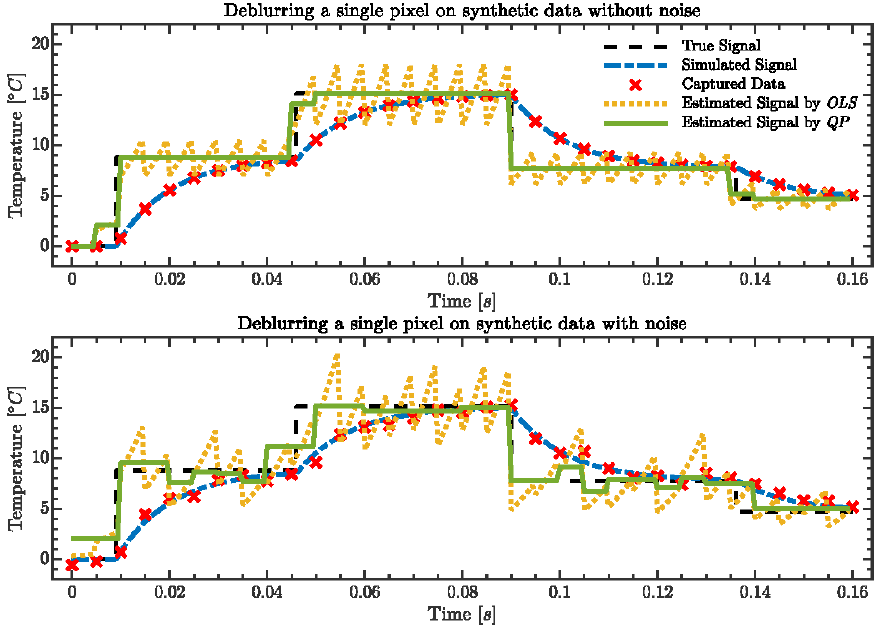}
    \caption{An illustration that the ordinary least squares solution is unable to capture the true underlying signal even without noise. Whereas the QP solution aligns well with the ground truth signal even after the introduction of noise.}
    \label{fig:single_pixel_deblur}
\end{figure}

In this experiment, we analyse the proposed method for deblurring a single pixel.
We generated a piece-wise constant signal and simulated the microbolometer response with $\tau=11\mathrm{ms}$ with additive Gaussian noise of standard deviation $0.5 \mathrm{K}$.
For comparison, we process the same signal with and without noise.
The Ordinary Least Squares (OLS) solution to \eqref{eq:thm_A_space} does not align well with the true signal as seen in Fig.~\ref{fig:single_pixel_deblur} even without noise.
While the true signal has $4$ transitions, the number of transitions in OLS is much higher.
However, the QP solution agrees well with the true signal in both the noisy and noiseless scenarios.
The $\lambda$ values used here were $0.001$ and $2$ for the noiseless and noisy signals, respectively. 
This demonstrates the utility of sparsity as presented in Section~\ref{sec:deblurring_equivalence}.

\begin{figure}[!ht]
    \centering
    \begin{tabular}{@{}c@{} c}
        \rotatebox[origin=c]{90}{\footnotesize{Blurry}} & \raisebox{-0.45\height}{\includegraphics[width=\linewidth]{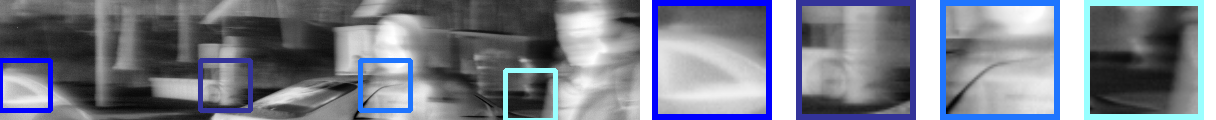}} \\[12pt]
        \rotatebox[origin=c]{90}{\footnotesize{ECP}} & \raisebox{-0.4\height}{\includegraphics[width=\linewidth]{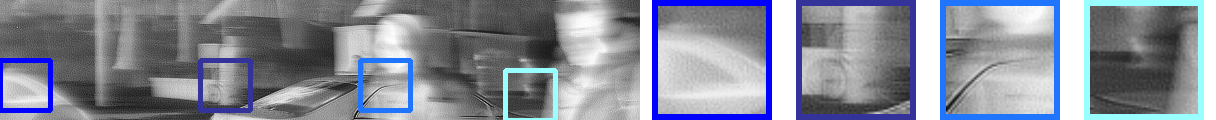}} \\[12pt]
        \rotatebox[origin=c]{90}{\footnotesize{SRN}} & \raisebox{-0.4\height}{\includegraphics[width=\linewidth]{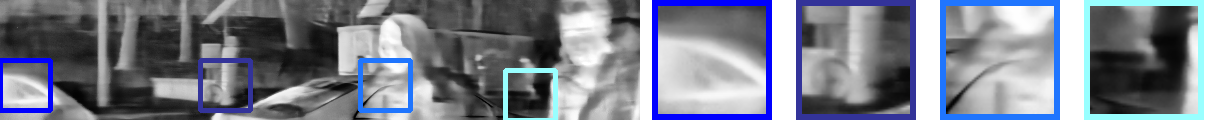}} \\[12pt]
        \rotatebox[origin=c]{90}{\footnotesize{DGAN}} & \raisebox{-0.4\height}{\includegraphics[width=\linewidth]{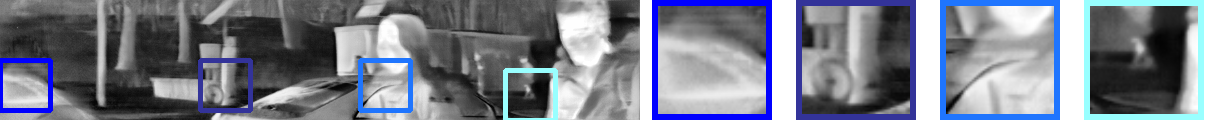}} \\[12pt]
        \rotatebox[origin=c]{90}{\footnotesize{Ours}} & \raisebox{-0.4\height}{\includegraphics[width=\linewidth]{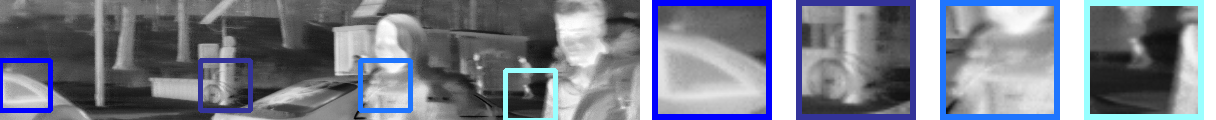}} \\
    \end{tabular}
    \caption{An illustration of the output of various deblurring algorithms on a thermal image (top row) of an outdoor scene with objects at different depths. 
    In particular, call out 2 (third column from left) illustrates a bicycle frame that is recovered by our approach while only the front wheel is recovered by some of the other approaches. Call out 3 (fourth column from left) depicts the deblurring result when the motion blur in the input image sequence blended the car and the person in front of it.}
    \label{fig:seq03_frame2141}
\end{figure}

\subsection{Dataset and Baselines}
\label{subsec:blurry_thermal_dataset}

We collected a test dataset, which we refer to as the \emph{Blurry Thermal Data set}, with generic camera motions in outdoor scenes with moving pedestrians and vehicles. 
Note that obtaining ground truth thermal images without motion blur using microbolometers is infeasible as there is no control over exposure time. 
One could use a cooled photon detector, but that would require pixel-wise correspondence between the microbolometer and the expensive cooled photon detector, which is a challenging task. 
Therefore, we are unable to quantitatively evaluate the deblurring algorithms directly. 

As a proxy, we evaluate the utility of such deblurring algorithms by benchmarking the performance of state-of-the-art object detectors on blurred and deblurred thermal images using our proposed algorithm and a variety of existing, state-of-the-art approaches. 
To this end, we annotated a subset of $406$ images using the COCO Annotator \cite{cocoannotator} for a total of $1168$ pedestrians and $831$ vehicles. 
We compared our approach against several classes of deblurring algorithms. Single-image deblurring algorithms: 1) non-learning based method (ECP~\cite{DBLP:conf/cvpr/YanRGWC17}), 2) learning based approaches (SRN~\cite{DBLP:conf/cvpr/TaoGSWJ18}, DGAN~\cite{DBLP:conf/iccv/KupynMWW19}). 
Video deblurring algorithms: 3) learning based methods (STFAN~\cite{DBLP:conf/iccv/ZhouZPZXR19}, EDVR~\cite{DBLP:conf/cvpr/WangCYDL19}). 
Since thermal images are in units of temperature, we map the temperature range of the images to $[0,1]$ and used the floating point representation directly, whenever possible, in the baseline algorithms to avoid loss of precision due to intermediate 8-bit image formats.

\subsection{Motion Deblurring of Natural Scenes}
\label{subsec:deblur_images}

This section qualitatively compares the performance of the various deblurring approaches. 
For brevity, we compare performance with the non-learning algorithm ECP and two strong learning-based algorithms, SRN and DGAN. 
Comparison with all baseline algorithms are presented as supplementary material for reference.
Figs.~\ref{fig:seq03_frame2141} and~\ref{fig:seq04_frame2610} illustrate images captured at an intersection with pedestrians and cars moving in different directions at different depths causing highly non-uniform blur.
The non-learning based, ECP algorithm merely sharpens the image without removing any blur as it assumes uniform blur kernel.
In all the call outs, our proposed model-based approach achieves equal or superior deblurring without any artifacts suffered by learning based algorithms.

\begin{figure}[!ht]
    \centering
        \begin{tabular}{@{}c@{} c}
\rotatebox[origin=c]{90}{\footnotesize{Blurry}} & \raisebox{-0.45\height}{\includegraphics[width=\linewidth]{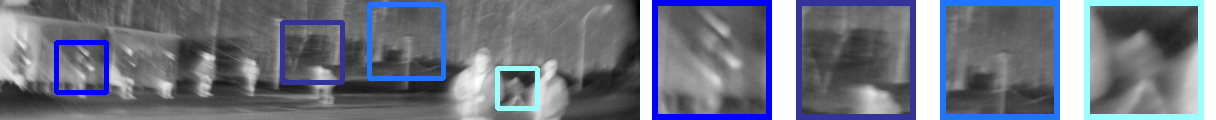}} \\[12pt]
        \rotatebox[origin=c]{90}{\footnotesize{ECP}} & \raisebox{-0.4\height}{\includegraphics[width=\linewidth]{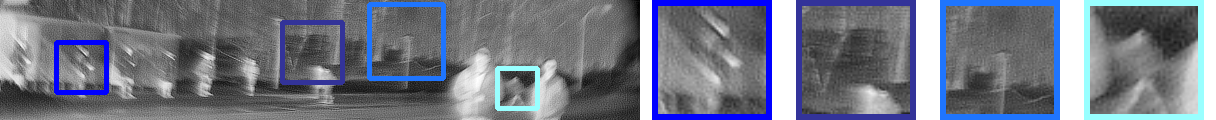}} \\[12pt]
        \rotatebox[origin=c]{90}{\footnotesize{SRN}} & \raisebox{-0.4\height}{\includegraphics[width=\linewidth]{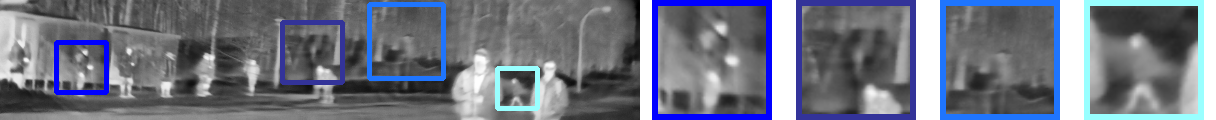}} \\[12pt]
        \rotatebox[origin=c]{90}{\footnotesize{DGAN}} & \raisebox{-0.4\height}{\includegraphics[width=\linewidth]{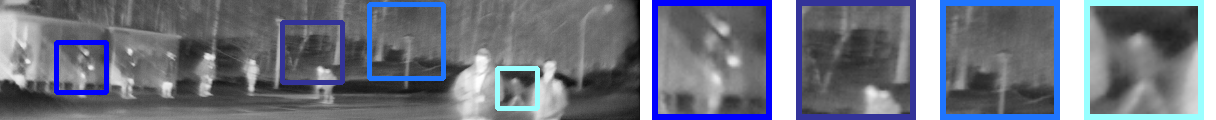}} \\[12pt]
        \rotatebox[origin=c]{90}{\footnotesize{Ours}} & \raisebox{-0.4\height}{\includegraphics[width=\linewidth]{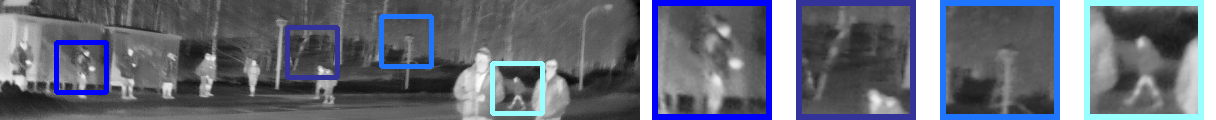}} \\
    \end{tabular}
    \caption{An illustration of the output of various deblurring algorithms on a thermal image (top row) when an arbitrary camera motion.
    ECP sharpens the image without removing blur. 
    SRN introduces artifacts seen in call outs 2 and 3 (columns 3 and 4 from the left, respectively). 
    DGAN has residual blur in call out 4 (column 5 from the left).
    In  comparison our proposed approach shows no artifacts or blur.}
    \label{fig:seq04_frame2610}
\end{figure}

\subsection{Object Detection under Severe Motion Blur}
\label{subsec:object_detection}

\begin{figure}[!ht]
    \centering
    \begin{tabular}{@{}c@{} c}
        \rotatebox[origin=c]{90}{\footnotesize{Blurry}} & \raisebox{-0.5\height}{\includegraphics[width=0.95\linewidth]{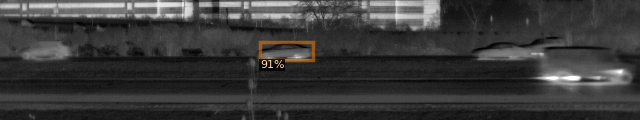}} \\[25pt]
        \rotatebox[origin=c]{90}{\footnotesize{ECP}} & \raisebox{-0.5\height}{\includegraphics[width=0.95\linewidth]{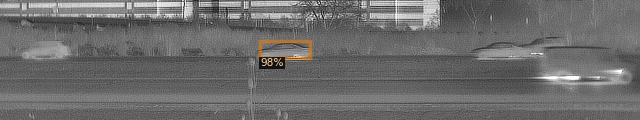}} \\[25pt]
        \rotatebox[origin=c]{90}{\footnotesize{SRN}} & \raisebox{-0.5\height}{\includegraphics[width=0.95\linewidth]{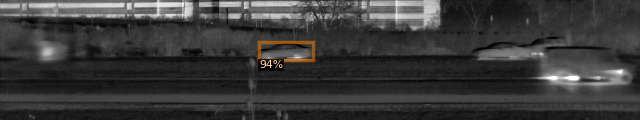}} \\[25pt]
        \rotatebox[origin=c]{90}{\footnotesize{DGAN}} & \raisebox{-0.5\height}{\includegraphics[width=0.95\linewidth]{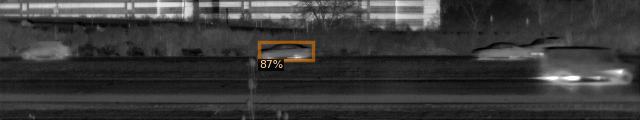}} \\[25pt]
        \rotatebox[origin=c]{90}{\footnotesize{Ours}} & \raisebox{-0.5\height}{\includegraphics[width=0.95\linewidth]{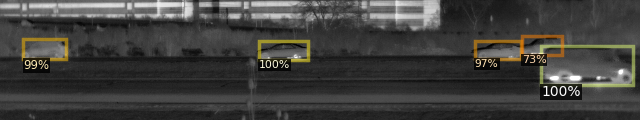}} \\[25pt]
    \end{tabular}
    \caption{An illustration of object detection results from RetinaNet using the original image, the output of various state-of-the-art deblurring techniques, and our proposed approach. The stationary camera is viewing an expressway with fast moving cars (top image). The images deblurred using our proposed approach enables the detector to find all the cars in the scene.}
    \label{fig:seq05_frame22476}
\end{figure}

\begin{table}[!t]
\centering
\begin{tabular}{|l|l|c|c|c|c|}
\hline
\textbf{Detector} & \textbf{Image Type} & \textbf{Person} & \textbf{Car} & \textbf{Total} & \textbf{$\Delta$AP}\\
\hline
\multirow{7}{4em}{Faster R-CNN}
 & Blurred & 13.62 & 26.88 & 20.25 & -\\
 \cline{2-6}
 & ECP \cite{DBLP:conf/cvpr/YanRGWC17} & 11.27 & 24.11 & 17.69 & -2.56 \\
 & STFAN \cite{DBLP:conf/iccv/ZhouZPZXR19}& 15.18 & 28.23 & 21.71 & 1.46 \\
 & EDVR \cite{DBLP:conf/cvpr/WangCYDL19} & 14.80 & 29.25 & 22.02 & 1.77 \\
 & SRN \cite{DBLP:conf/cvpr/TaoGSWJ18} & 17.46 & 30.99 & 24.23 & 3.98 \\
 & DGAN \cite{DBLP:conf/iccv/KupynMWW19} & 16.94 & 31.40 & 24.16 & 3.91 \\
 & \textbf{Ours} & \textbf{21.81} & \textbf{40.53} & \textbf{31.17} & \textbf{10.92} \\
\hline
\hline

\multirow{7}{4em}{RetinaNet}
 & Blurred & 12.99 & 29.92 & 21.46 & - \\
 \cline{2-6}
 & ECP \cite{DBLP:conf/cvpr/YanRGWC17} & 10.60 & 28.08 & 19.34 & -2.12 \\
 & STFAN \cite{DBLP:conf/iccv/ZhouZPZXR19} & 14.69 & 30.20 & 22.44 & 0.98 \\
 & EDVR \cite{DBLP:conf/cvpr/WangCYDL19} & 14.32 & 31.61 & 22.96 & 1.5 \\
 & SRN \cite{DBLP:conf/cvpr/TaoGSWJ18} & 14.99 & 31.96 & 23.48 & 2.02 \\
 & DGAN \cite{DBLP:conf/iccv/KupynMWW19} & 16.18 & 32.84 & 24.51 & 3.05 \\
 & \textbf{Ours} & \textbf{20.61} & \textbf{41.15} & \textbf{30.88} & \textbf{9.42} \\
\hline
\hline

\multirow{7}{4em}{YoloV3}
 & Blurred & 13.14 & 24.80 & 18.97 & - \\
 \cline{2-6}
 & ECP \cite{DBLP:conf/cvpr/YanRGWC17} & 12.35 & 21.68 & 17.02 & -1.95 \\
 & STFAN \cite{DBLP:conf/iccv/ZhouZPZXR19} & 14.10 & 25.58 & 19.84 & 0.87 \\
 & EDVR \cite{DBLP:conf/cvpr/WangCYDL19} & 14.22 & 26.03 & 20.13 & 1.16 \\
 & SRN \cite{DBLP:conf/cvpr/TaoGSWJ18} & 15.93 & 26.73 & 21.33 & 2.36 \\
 & DGAN \cite{DBLP:conf/iccv/KupynMWW19} & 15.20 & 27.64 & 21.42 & 2.45 \\
 & \textbf{Ours} & \textbf{18.63} & \textbf{34.55} & \textbf{26.59} & \textbf{7.62} \\
\hline
\end{tabular}
\caption{The Average Precision scores for three state-of-the-art object detectors when using blurred images or images deblurred by our proposed approach and various state-of-the-art techniques. 
$\Delta AP$ shows the change in performance between using deblurred images and the original blurred images. Deblurring using our proposed algorithm increases performance by more than twice that achieved by using any tested, state-of-the-art deblurring methods.}
\label{tbl:mAP_detection}
\end{table}

\begin{figure}[ht]
    \centering
    \begin{tabular}{@{}c@{} c}
        \rotatebox[origin=c]{90}{\footnotesize{Blurry}} & \raisebox{-0.5\height}{\includegraphics[width=0.95\linewidth]{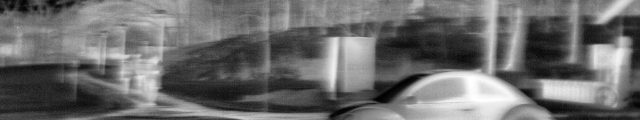}} \\[25pt]
        \rotatebox[origin=c]{90}{\footnotesize{ECP}} & \raisebox{-0.5\height}{\includegraphics[width=0.95\linewidth]{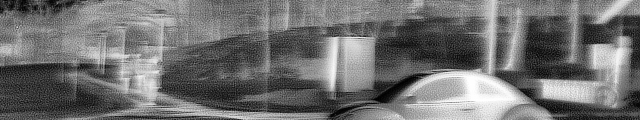}} \\[25pt]
        \rotatebox[origin=c]{90}{\footnotesize{SRN}} & \raisebox{-0.5\height}{\includegraphics[width=0.95\linewidth]{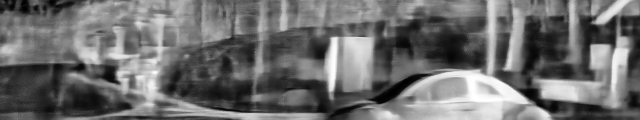}} \\[25pt]
        \rotatebox[origin=c]{90}{\footnotesize{DGAN}} & \raisebox{-0.5\height}{\includegraphics[width=0.95\linewidth]{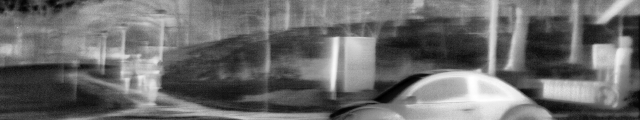}} \\[25pt]
        \rotatebox[origin=c]{90}{\footnotesize{Ours}} & \raisebox{-0.5\height}{\includegraphics[width=0.95\linewidth]{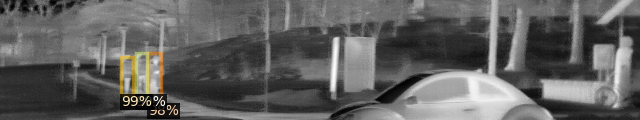}} \\[25pt]
    \end{tabular}
    \caption{An illustration of the object detection results from Faster R-CNN using the original image, the output of various state-of-the-art deblurring techniques, and our proposed aproach. 
    SRN introduces large artefacts while residual blur is still visible in the DGAN and ECP results.
    Our proposed approach eliminates blur and aids the detector to identify all three pedestrians in the image which were otherwise missed.}
    \label{fig:seq03_frame1957}
\end{figure}

We consider three state-of-the-art object detectors, namely Faster R-CNN~\cite{DBLP:journals/pami/RenHG017}, RetinaNet~\cite{DBLP:journals/pami/LinGGHD20} and Yolov3~\cite{DBLP:journals/corr/abs-1804-02767}.  
For Faster R-CNN and RetinaNet, we use their state-of-the-art implementation in Detectron2~\cite{wu2019detectron2} with Feature Pyramid Networks~\cite{DBLP:conf/cvpr/LinDGHHB17} and a ResNet-101~\cite{DBLP:conf/cvpr/HeZRS16} backbone.
The object detectors were trained on the FLIR Thermal Data set~\cite{fliradas} and their mean Average Precision(mAP) scores on the corresponding FLIR Thermal validation set were $46.57\%$, $42.54\%$ and $38.38\%$ for object detection, respectively.
Note that we only evaluate for \texttt{Person} and \texttt{Car} categories with our Blurry Thermal Data set. 
All evaluations were done using the stringent  COCO~\cite{DBLP:conf/eccv/LinMBHPRDZ14} metric which takes average over Intersection-over-Union (IoU) thresholds every $0.05$ from $0.5$ to $0.95$.

The trained object detectors were tested on the Blurry Thermal Data set.
Fig.~\ref{fig:seq05_frame22476} depicts object detection results for RetinaNet on a highway scene. 
All the vehicles are detected when using images restored by our approach. 
Fig.~\ref{fig:seq03_frame1957} illustrates object detection results for Faster R-CNN for a test image containing multiple pedestrians approaching an intersection.
Our proposed approach aids the detector to find all three pedestrians while the other approaches were unable to deblur the pedestrians.
The performance of the detectors with the original blurred images and the images after deblurring using different algorithms are summarized in Table~\ref{tbl:mAP_detection}.

To summarize, the performance of all three detectors on the blurry images is significantly lower than their performance on the FLIR validation set.
This emphasizes the need for motion deblurring algorithms.
The non-learning based ECP generates deblurred images that actually degrade the performance of all three detectors when compared to just applying the blurred images directly to each of the detectors.
The learning-based algorithms improve object detection to various degrees with DGAN and SRN performing better than other baselines. 
However, our proposed approach more than doubles this improvement in Average Precision scores for all three detectors.
While using the strong baseline DGAN to restore blurred images increases the performance by $3.91\%$, $3.05\%$ and $2.45\%$ for Faster R-CNN, RetinaNet and YoloV3, respectively. 
Using our proposed approach to restore blurred images increases the performance by
$10.92\%$, $9.42\%$ and $7.62\%$ for Faster R-CNN, RetinaNet and YoloV3, respectively.

\section{Discussion and Conclusion}
\label{sec:discussion}

Motion is inherent to mobile robots where the world continuously evolves in time. 
In this context, microbolometers provide a unique sensing modality as they are constantly exposed to the scene, accumulating the evolution in time. 
Until now, the thermal inertia of microbolometers has largely been considered as an unwanted phenomenon that needs to be minimized if not eliminated. 
Our work shows that spreading the information in time can in-fact be leveraged using intelligent algorithms to side-step the prominent challenges in motion deblurring - 6D relative motion and dynamic objects in the scene. 
The pixel-wise formulation of our approach allows parallel processing to be utilised to speed up motion deblurring. 
Furthermore, a sparse set of keypoints could be deblurred on-demand providing an attractive utility for robotic vision. 

Motion blur in thermal cameras is a primary impediment for their wider adoption in robotics despite their ability to see in the dark.
The lack of exposure control and global shutter features further aggravates the problem.
In this work, we showed that the motion blur in thermal cameras is caused due to the thermal inertia of each pixel that is unique to the physics of image formation in these cameras.
We showed that reversing this effect amounts to solving an underdetermined system of linear equations which have infinitely many solutions.
By leveraging high frame rate and temporal sparsity, we formulated LASSO optimization problem and solved it using Quadratic Programming(QP) to estimate the incident power at each pixel independently and yet recovered motion deblurred images that are spatially coherent.
Our work shows that just temporal sparsity of signals is sufficient to achieve motion deblurring without modeling the relative motion or scene dynamics.

\section*{Acknowledgments}
This work was supported by a grant from Ford Motor Company via the Ford-UM Alliance under award N022884. The authors would also like to thank the reviewers for their useful comments and suggestions.

%% Use plainnat to work nicely with natbib. 

\bibliographystyle{plainnat}
\bibliography{references}

\end{document}

\title{Supplementary Materials for \\ ``Pixel-Wise Motion Deblurring of Thermal Videos''}

% You will get a Paper-ID when submitting a pdf file to the conference system
\author{Paper ID: 1238. Author Names Omitted for Anonymous Review.}

%\author{\authorblockN{Michael Shell}
%\authorblockA{School of Electrical and\\Computer Engineering\\
%Georgia Institute of Technology\\
%Atlanta, Georgia 30332--0250\\
%Email: mshell@ece.gatech.edu}
%\and
%\authorblockA{Twentieth Century Fox\\
%\and
%\authorblockN{James Kirk\\ and Montgomery Scott}
%\authorblockA{Starfleet Academy\\
%San Francisco, California 96678-2391\\

% avoiding spaces at the end of the author lines is not a problem with
% conference papers because we don't use \thanks or \IEEEmembership

%Homer Simpson\authorrefmark{2},
%Montgomery Scott\authorrefmark{3} and
%Eldon Tyrell\authorrefmark{4}}
%Telephone: (800) 555--1212, Fax: (888) 555--1212}
%\authorblockA{\authorrefmark{4}Tyrell Inc., 123 Replicant Street, Los Angeles, California 90210--4321}}

\maketitle

\IEEEpeerreviewmaketitle

\section{Additional Results}

\begin{figure}[b]
    \centering
    \includegraphics[width=\linewidth]{{suppl_materials_figures/seq_03_001897.eps}}
    \caption{This image was captured using a panning motion of the camera at an intersection with little scene dynamics. In this image, both SRN and DGAN introduce artefacts that deform the image arbitrarily. In particular, note the bike wheel in the result of DGAN and the rear wheel of the car.}
    \label{fig: scene_3_bike_artifact}
\end{figure}

\begin{figure}[b]
    \centering
    \includegraphics[width=\linewidth]{{suppl_materials_figures/seq_01_001686.eps}}
    \caption{This image was captured using a panning motion of the camera viewing a largely static scene. The figure illustrates the strong artefacts introduced by DGAN on images from the Blurry Thermal Dataset. }
    \label{fig:scene_1_DGAN_fail}
\end{figure}
Our work presents a novel formulation of motion deblurring that is specific to the physics of image formation in uncooled microbolometers. 
In the main paper, we presented qualitative comparison with one non-learning algorithm \cite{yan2017image} and two strong learning based algorithms \cite{ kupyn2019deblurgan, tao2018scale} due to space constraints. 
Here, we present additional comparisons with all considered baseline algorithms. 
Note that there are no ground truth datasets for motion blur in thermal imaging.
The lack of both global shutter and the control over exposure time, as described in the main paper, emphasize the difficulty in obtaining such ground truth datasets.
Therefore, pre-trained weights were used for the learning based deblurring algorithms compared here. 
In particular, we include results from STFAN~\cite{zhou2019stfan} and EDVR~\cite{wang2019edvr} which are video deblurring algorithms and hence have access to temporal data.
STFAN uses $2$ temporally adjacent frames while EDVR uses $5$ temporally adjacent images. 
The superior performance of our proposed approach over video deblurring algorithms illustrates the utility of processing each pixel independently.

We show qualitative comparison of all baselines with our proposed approach on five different scenes (Figures [1-5]). 
The performance of ECP, SRN and DGAN on these scenes are consistent with the results presented in the main paper. 
In more detail, ECP sharpens the image without reducing blur. 
While SRN and DGAN are able to deblur some scenes, they can introduce strong artefacts as seen clearly from Figure~\ref{fig: scene_4_DGAN_fail} and Figure~\ref{fig:scene_1_DGAN_fail}.
The video deblurring algorithms tested here, STFAN and EDVR, do not show any visible decrease in blur in a large number of tested scenes.
However, our proposed approach consistently performs motion deblurring without incurring any visible artefacts.

%% Use plainnat to work nicely with natbib. 

\begin{figure}[t]
    \centering
    \includegraphics[width=\linewidth]{{suppl_materials_figures/seq_06_005258.eps}}
    \caption{This image was captured in a typical driving scene. The motion blur is particularly visible on the car. Our result completely removes the motion blur while residual blur is visible in all other baseline algorithms.}
    \label{fig: scene_6_car_line_artifacts}
\end{figure}

\begin{figure}[t]
    \centering
    \includegraphics[width=\linewidth]{{suppl_materials_figures/seq_04_005221.eps}}
    \caption{This image was captured using a rolling motion of the camera at an intersection. While DGAN and SRN introduce strong artefacts, other baselines are unable to achieve deblurring as well. }
    \label{fig: scene_4_DGAN_fail}
\end{figure}

\begin{figure}[t]
    \centering
    \includegraphics[width=\linewidth]{{suppl_materials_figures/seq_04_003440.eps}}
    \caption{This image was captured using a rolling motion of the camera that introduces stronger motion along the left and right extremes of the image. Note the varying degrees of motion blur on the car at the center and the pedestrians at different depths to the left of the image. Our proposed approach successfully deblurs all parts of the image. }
    \label{fig: scene_4_roll_motion}
\end{figure}
\clearpage
\bibliographystyle{plainnat}
\bibliography{references}